%% file: main.tex
\documentclass{article}

\usepackage{PRIMEarxiv}

\usepackage[utf8]{inputenc}
\usepackage[T1]{fontenc}
\usepackage{url}

\input{math_commands.tex}
\input{preamble}

\usepackage{microtype}
\usepackage{hyperref}

\graphicspath{{figures/}}

\title{RODR: Riemannian Orthogonally Decoupled Regularization for Disentangled Manifold Representation}

\author[1]{Jiayu Zhu, Wenlai Zhao \thanks{Email: zhujiayu24@mails.tsinghua.edu.cn, zhaowenlai@tsinghua.edu.cn}}

\affil[ ]{Department of Computer Science and Technology, Tsinghua University, Beijing, 100084, China}

\date{July 24,2026} 

\begin{document}
\maketitle

\input{sec/0_abstract}
\input{sec/1_intro}

\input{sec/2_related}
\input{sec/3_problem}
\input{sec/4_method}
\input{sec/5_Theoretical}
\input{sec/6_experiment}

\input{sec/7_conclusion}

\bibliographystyle{unsrt}
\bibliography{references}

\input{sec/8_appendix}
\end{document}

%% file: math_commands.tex
\usepackage{amsmath,amsfonts,bm}

\def\eqref#1{equation~\ref{#1}}

\def\1{\bm{1}}

\DeclareMathAlphabet{\mathsfit}{\encodingdefault}{\sfdefault}{m}{sl}
\SetMathAlphabet{\mathsfit}{bold}{\encodingdefault}{\sfdefault}{bx}{n}

%% file: preamble.tex
\usepackage{amsmath,amssymb,amsthm}
\usepackage{bm}
\usepackage{mathtools}
\usepackage{graphicx}
\usepackage{booktabs}
\usepackage{enumitem}
\usepackage{float}
\usepackage{algorithm}
\usepackage{algorithmic}
\usepackage{tikz}
\usepackage{multirow}
\usepackage{subcaption}
\usepackage{xcolor}
\usepackage{authblk}
\usetikzlibrary{arrows.meta,angles,quotes,calc}

\newtheorem{lemma}{Lemma}[section]
\newtheorem{theorem}{Theorem}[section]

\newtheorem{proposition}{Proposition}[section]



%% file: sec/0_abstract.tex
\begin{abstract}
Point cloud denoising is essentially a geometric recovery task that aims to reconstruct the intrinsic structure of a smooth 2D Riemannian manifold embedded in $\mathbb{R}^{3}$ from noisy, discrete ambient-space samples. Despite the remarkable progress of modern manifold-aware encoders and generative transport models in geometric representation learning, a fundamental objective-geometry mismatch remains underexplored. Theoretically, we identified that this mismatched coupling leads to \textit{geometric gradient interference}, where conflicting optimization objectives result in \textit{structural degradation} and \textit{point clustering}. We introduce \textbf{Riemannian Orthogonally Decoupled Regularization (RODR)} to reformulate the optimization trajectory by disentangling the normal (fitting) and tangential (distribution) components. Guided by a vector-attention and entropy-aware adaptive strategy, RODR effectively preserves high-fidelity geometric details while maintaining sampling uniformity. Experiments demonstrate that RODR reaches performance comparable to state-of-the-art baselines and suggests improved distribution regularity and reduced local aggregation effectively. Our work establishes a generic and interpretable framework for disentangled geometric optimization in point cloud processing.
\end{abstract}

%% file: sec/1_intro.tex
\section{Introduction}
\label{sec:intro}

Point cloud is the dominant discrete geometric representation for 3D scenes, widely deployed in autonomous perception, robotic reconstruction, and digital twin systems. The entire pipeline of point cloud restoration fundamentally relies on the \textit{Riemannian manifold hypothesis}: noise-free real-world surfaces correspond to smooth 2D Riemannian manifold  $(\mathcal{M},g)$ embedded in the ambient space $\mathbb{R}^{3}$, and clean point clouds are discrete uniform samplings of such continuous manifolds. Modern point cloud denoising methods have evolved from naive coordinate regression to advanced geometric learning paradigms, including manifold-aware encoding, score-based diffusion~\cite{cai2021score}, and flow-matching generative models~\cite{jiang2024straightening}, all of which explicitly pursue compliance with the intrinsic geometric properties of underlying surfaces.

Different from general multi-task gradient conflicts caused by task incompatibility, the interference in point cloud manifold learning is \textit{rooted in the structural entanglement of Euclidean ambient space and manifold intrinsic geometry}, which is a critical geometry-optimization inconsistency that plagues nearly all existing learning-based denoising frameworks. Manifold differential geometry frames point cloud denoising as a problem with two inherently orthogonal geometric optimization targets: extrinsic surface fitting that eliminates off-manifold noise along the normal bundle $N\mathcal{M}$, and intrinsic sampling regularization that optimizes in-plane point distribution along the tangent bundle $T\mathcal{M}$. 

However, current optimization objectives uniformly adopt coupled Euclidean losses in ambient space, mixing extrinsic and intrinsic geometric updates into a single 3D gradient vector without geometric partitioning. In this coupled optimization paradigm, fitting gradients for surface fidelity and diffusion gradients for sampling uniformity produce non-trivial inner-product interference, leading to two typical pathological optimization behaviors: negative gradient interference counteracts geometric update signals, freezing points in high-curvature regions and slowing convergence; positive interference amplifies ambient displacement, pushing points away from the true manifold and generating ghosting artifacts.
 
To address this intrinsic geometric limitation, we rethink the optimization dynamics of manifold restoration and propose \textit{\textbf{Riemannian Orthogonally Decoupled Regularization (RODR)}}, a geometrically principled gradient optimization framework tailored for manifold representation learning. The core insight of RODR is to enforce strict geometric assignment via subspace orthogonal projection: extrinsic surface fitting is exclusively optimized on the normal bundle for precise surface recovery, while intrinsic sampling uniformity is independently regularized on the tangent bundle for reasonable in-plane distribution. This geometric gradient surgery effectively decouples the conflicting optimization objectives, eliminates cross-subspace gradient interference, and transforms the ill-conditioned Euclidean optimization landscape into a geometrically well-posed system. Furthermore, we design a vector-attention and entropy-aware adaptive strategy mechanism to dynamically modulate optimization intensity according to local surface geometric features, mitigating the approximation bias of tangent-plane proxies in high-curvature regions and preserving fine-grained geometric details.

\begin{figure}[t]
  \centering
  \includegraphics[width=\linewidth]{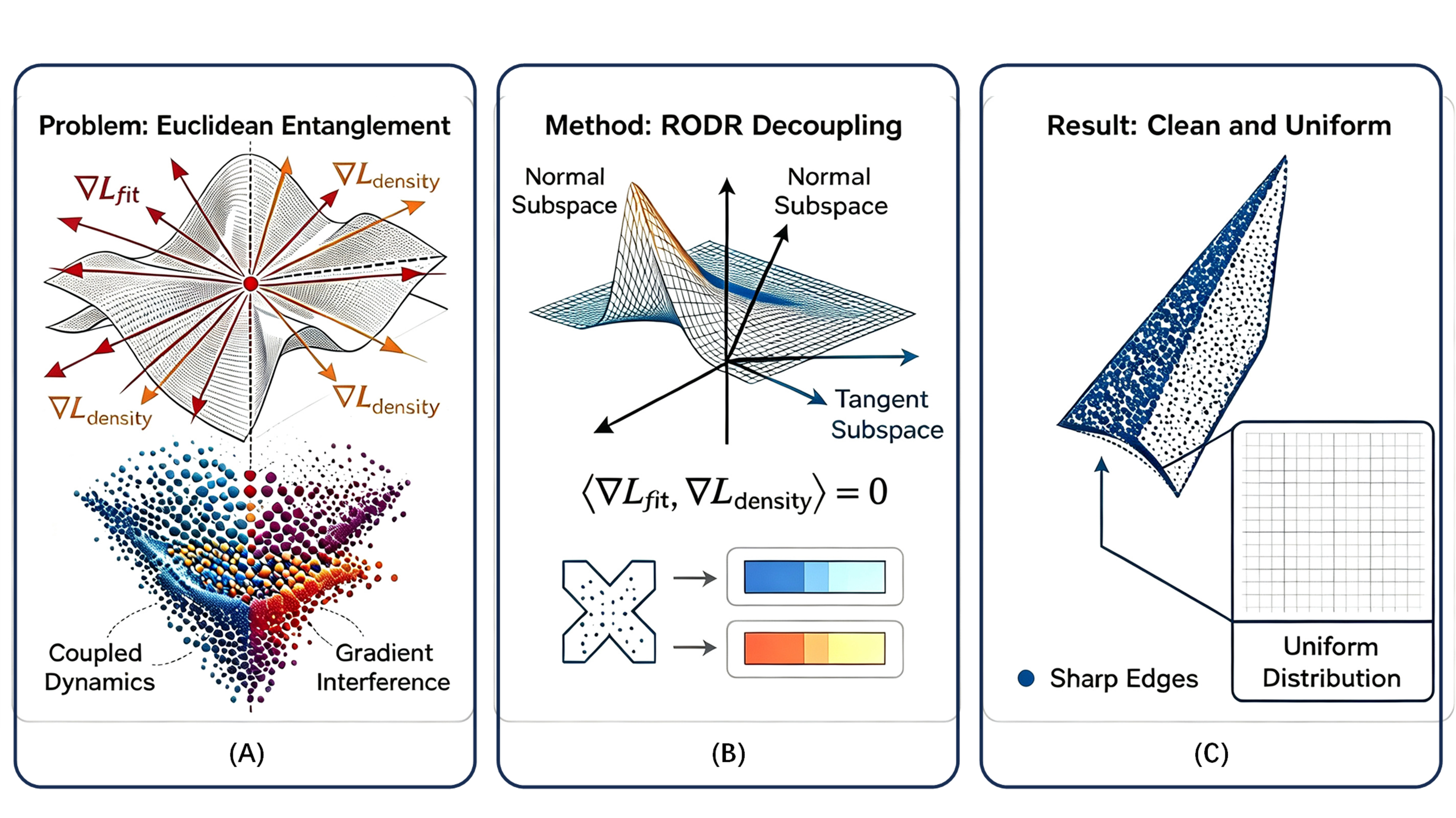}
  \caption{\textbf{Conceptual overview of the proposed RODR framework.} (A) Euclidean regularization couples fitting and density repulsion losses within a unified ambient space, resulting in conflicting tangential gradient components, point clumping, and excessive smoothing of sharp geometric features. (B) The RODR framework exploits intrinsic manifold geometry to project the fitting objective onto the normal subspace and the uniform-density objective onto the tangent subspace, yielding strictly orthogonal optimization gradients with a zero inner product. (C) Visual denoising comparison on sampled point clouds. The baseline method suffers from severe point clumping and degraded sharp edges, whereas RODR simultaneously achieves uniform point distribution and precise preservation of delicate geometric contours.}
  \label{fig:teaser}
\end{figure}

In summary, our work delivers three major technical and theoretical contributions:
\begin{itemize}
\item \textit{Geometric Mechanism Analysis.} We theoretically reveal the inherent gradient interference problem caused by Euclidean space coupling in manifold restoration, systematically explaining the formation mechanism of representation collapse, point clumping and feature erosion, and distinguishing geometric-induced interference from conventional multi-task gradient conflicts.
\item \textit{Orthogonal Decoupling Framework.} We propose RODR, a novel Riemannian geometric regularization method that realizes complete extrinsic-intrinsic optimization decoupling via tangent-normal bundle orthogonal projection, fundamentally reshaping the optimization dynamics of manifold representation learning.
\item \textit{Superior Practical and Theoretical Performance.} Experiments on multiple benchmarks and datasets demonstrate that RODR achieves competitive performance with recent denoising methods with prominent feature preservation and uniform sampling. Comprehensive theoretical proofs further validate the stability, convergence and geometric consistency of the proposed framework.
\end{itemize}

In the broader context of \textit{representation learning}, it is well-established that coupling distinct objectives within a single gradient channel can lead to suboptimal optimization dynamics and interference~\cite{yu2020gradient}. As a lightweight, plug-and-play regularization module, RODR is compatible with mainstream regression-based and generative transport-based point cloud learning backbones without extra computational overhead. We further provide rigorous theoretical analysis including frame stability bound, Hessian block-diagonal preconditioning property, and asymptotic geometric consistency, fully verifying the rationality and robustness of the proposed decoupled optimization dynamics.

%% file: sec/2_related.tex
\section{Related Work}
\label{sec:related}

Point cloud denoising has been extensively studied through both traditional geometry-driven methods and modern learning-based paradigms. Classical approaches leverage local geometric priors, including bilateral filtering~\cite{fleishman2003bilateral}, moving least squares (MLS) projection~\cite{levin2004mesh}, Laplacian smoothing~\cite{taubin1995signal}, and parameterization-free projection (LOP)~\cite{lipman2007parameterization}. While these methods are intuitive and effective on relatively smooth surfaces, they typically rely on hand-crafted local assumptions that can be challenging to generalize to complex scenarios involving heavy noise or high-curvature features.

\subsection{Euclidean Coordinate Regression}
Deep learning-based denoising has advanced significantly from basic displacement learning to sophisticated score-function regression. Notable architectures such as PointCleanNet~\cite{hermosilla2020pointcleannet}, DUP-Net~\cite{wei2019dup}, PCDNF~\cite{luo2021pcdnf}, PointFilter~\cite{zhang2021pointfilter}, PD-Flow~\cite{mao2022pdflow}, DeepPSR~\cite{chen2023deeppsr}, and modern bilateral graph-based methods~\cite{ben20213d,wang2022curvature,sun2021bilateral} have established high standards for robustness and performance. Recent iterative and adaptive filtering methods further improve denoising quality by learning when or how to stop refinement, including IterativePFN~\cite{desilva2023iterativepfn}, ASDN~\cite{guo2025asdn}, 3DMambaIPF~\cite{zhou2025mambaipf}, DSNet~\cite{cheng2026dsnet}, and PointInDI~\cite{liu2026pointindi}. These methods predominantly treat denoising as a coordinate-wise regression task in the ambient space $\mathbb{R}^3$. In this setting, surface fitting and point distribution are often optimized through joint Euclidean regularization.

In practical implementations, such Euclidean formulations effectively guide noisy points toward local geometric proxies, such as planes or learned targets. While this strategy substantially improves the structural fidelity of the point cloud, it presents an inherent coupling between fitting (normal direction) and uniformity (tangential direction) objectives. As these components are often optimized through a unified gradient in $\mathbb{R}^3$, achieving an optimal balance between surface convergence and distribution regularity typically requires delicate hyperparameter tuning~\cite{sun2021bilateral, qi2017pointnetplusplus}. This coupling provides a motivation to explore decoupled optimization strategies that can more independently preserve sharp features while maintaining sampling uniformity.

\subsection{Riemannian Manifold-Aware Encoders}
The integration of Riemannian geometry into 3D vision~\cite{liu2021riemannian} is rooted in the manifold hypothesis, which posits that noise-free point clouds are discrete samples from a smooth 2D Riemannian manifold $(\mathcal{M}, g)$ embedded in $\mathbb{R}^3$. Recent theoretical studies further examine manifold fitting under unbounded noise~\cite{yao2025manifoldfitting} and curvature-driven recovery under unbounded isotropic noise~\cite{li2026curvaturedriven}, providing useful geometric context for robust surface reconstruction. Substantial progress has been made in enhancing networks' structural awareness of this underlying geometry. Modern backbones, including DGCNN~\cite{wang2019dgcnn}, Point Transformer~\cite{zhao2021pointtransformer}, and $SE(3)$-equivariant GNNs~\cite{satorras2021equivariant,chen2022equivariant}, successfully incorporate local geometric priors into the feature extraction process. Furthermore, recent implicit methods like Neural-Pull~\cite{ma2021neuralpull} and its successors~\cite{ma2023refined} represent the manifold more continuously through learned signed distance fields (SDFs).

While these architectural innovations have significantly improved geometric representation, the supervision signals in many pipelines often rely on coupled Euclidean metrics. This suggests a potential for further synergy between the manifold-aware capability of the modern encoder and the design of the optimization objective. Specifically, there is an opportunity to develop loss functions that more explicitly reflect the intrinsic properties of the 2D manifold, thereby enabling more precise control over surface fidelity independently of the sampling distribution.

\subsection{Trajectory and Dynamic Optimization}
Recently, point cloud denoising has been reframed as a transport problem, mapping a noisy distribution to the target manifold distribution. Score-based models~\cite{cai2021score,zhou2023diffusion,luo2021scoredenoise} learn a vector field $\mathbf{v} = \nabla \log p(x)$ to guide points via reverse diffusion processes. This paradigm has been further refined by works such as Straight-PCF~\cite{jiang2024straightening}, P2P-Bridge~\cite{vogel2024p2pbridge}, Noise2Score3D~\cite{wei2025noise2score3d}, and adaptive score-diffusion filtering~\cite{wang2025adaptiveiterative}, which emphasize learned or analytically motivated trajectories in $\mathbb{R}^3$ to improve inference efficiency and robustness.

Mathematically, these generative fields represent an integrated transport operator that simultaneously attracts points toward the surface and regulates their lateral distribution. Since these two physical processes are characterized by a single $\mathbb{R}^3$ vector, the resulting flows must manage the inherent trade-off between geometric convergence and tangential diffusion. Recent multi-stage, multi-view, and learning-free designs such as BSV-PCD~\cite{chen2026bsvpcd}, DN-PCD~\cite{sheng2025dnpcd}, and NABNP~\cite{janowski2026nabnp} also highlight the continuing importance of robustness under low-quality or non-uniform point observations. Despite the impressive speed and robustness of trajectory-based models, addressing the entanglement of these objectives within the transport field remains a valuable direction for enhancing performance on complex, high-curvature geometries.

%% file: sec/3_problem.tex
\section{Representational Analysis and Motivation}
\label{sec:representation_analysis}

We model a noise-free point cloud as a discrete sampling of a smooth, two-dimensional Riemannian manifold $(\mathcal{M},g)$ embedded in the ambient space $\mathbb{R}^{3}$. The observed noisy points $\{p_i\}_{i=1}^N$ are generated via $p_i = x_i + \epsilon_i$, where $x_i \in \mathcal{M}$ and $\epsilon_i$ represents additive noise. The core objective of denoising is to learn a representation field that maps the noisy distribution back to the manifold $\mathcal{M}$.

\subsection{Representational Entanglement in Ambient Space}
\label{sec:entanglement}
Most learning-based denoising objectives are formulated as a joint Euclidean loss in the ambient space:
\begin{equation}
  L_{\mathrm{euc}}=L_{\mathrm{fit}}+\beta L_{\mathrm{uni}},\qquad \nabla_{\theta} L_{euc} = \sum_i \frac{\partial L_{euc}}{\partial p_i} \frac{\partial p_i}{\partial \theta}.
  \label{eq:euc_joint_loss}
\end{equation}
While efficient, this formulation introduces a fundamental \textit{representational entanglement}. For an embedded surface, the ambient $\mathbb{R}^3$ space decomposes into two geometrically distinct roles: the normal direction corrects extrinsic surface residuals, while the tangent directions control the intrinsic sampling pattern. In Eq.~\eqref{eq:euc_joint_loss}, these two components are entangled within a single 3D gradient channel. Consequently, the network is forced to learn a coupled representation where surface attraction and distribution diffusion interfere with each other, leading to suboptimal optimization trajectories.

\subsection{The Contractive Bias: Explaining Representation Collapse}
\label{sec:contractive_bias}
We argue that the common "clumping" artifact is not a random numerical error, but a result of a \textit{contractive bias} inherent in Euclidean fitting losses. Consider a point $p_i$ pulled toward a temporally frozen local geometric proxy $\mu_i$. Under a standard gradient step with rate $\eta$, the update $p_i^{t+1} \approx (1-\eta)p_i^t + \eta \mu_i^t$ acts as a local contraction. For two neighboring points $p_i, p_j$ sharing a correlated proxy, their pairwise distance evolves as:
\begin{equation}
  \|p_i^{t+1}-p_j^{t+1}\|_2^2 \approx (1-\eta)^2 \|p_i^t-p_j^t\|_2^2.
  \label{eq:contraction_mode}
\end{equation}
This indicates that the Euclidean fitting objective possesses an inherent \textbf{tangential contraction mode}, pulling points toward local centroids faster than any spacing regularizer can redistribute them. In the language of representation learning, this leads to a \textit{representation collapse} where the intrinsic geometric structure is destroyed by the extrinsic fitting force.

\subsection{Optimization Dynamics and Gradient Interference}
\label{sec:gradient_interference}
To characterize this entanglement, we define the \textbf{Gradient Interference} $C_i$ between surface fitting and distribution regularizer:
\begin{equation}
  C_i = \left\langle \nabla_{p_i} L_{\mathrm{fit}}, \nabla_{p_i} L_{\mathrm{uni}} \right\rangle.
  \label{eq:gradient_conflict}
\end{equation}
In the coupled Euclidean setting, there is no geometric constraint ensuring $C_i = 0$. This leads to critical optimization defects:
\begin{itemize}
    \item \textbf{Gradient Conflict ($C_i < 0$):} The fitting and uniformity forces counteract each other, effectively "freezing" the points in high-curvature regions and causing slow convergence.
    \item \textbf{Dynamic Distortion ($C_i > 0$):} The spacing term reinforcements the ambient displacement, unintentionally pushing points off the true surface to satisfy uniformity, causing "ghosting" artifacts.
\end{itemize}
This phenomenon is analogous to gradient interference in multi-task learning~\cite{yu2020gradient}. However, unlike general task-level conflicts, this interference is rooted in the \textit{physical geometry} of the manifold. 

\subsection{The Case for Orthogonal Decoupling}
\label{sec:decoupling_motivation}
The core issue is the absence of a \textbf{Geometric Assignment}. Surface fidelity is an extrinsic property of the manifold's \textit{normal bundle} $N\mathcal{M}$, while sampling uniformity is an intrinsic property of the \textit{tangent bundle} $T\mathcal{M}$. We posit that orthogonal decoupling is a necessary condition for disentangling the optimization dynamics. By constraining these forces into their respective subspaces, we can eliminate $C_i$ and transform the ill-conditioned Euclidean landscape into a geometrically well-posed one, enabling the independent learning of surface geometry and sampling distribution.

%% file: sec/4_method.tex
\section{RODR: Riemannian Orthogonally Decoupled Regularization}
\label{sec:method}

To resolve the representational entanglement identified in Section~\ref{sec:representation_analysis}, we propose \textbf{RODR}, a principled optimization framework that reshapes the learning dynamics via orthogonal subspace projection. RODR acts as a \textit{geometric preprocessor} that ensures surface attraction and sampling regularity are optimized in mutually exclusive channels.

\subsection{Geometric Gradient Surgery via Subspace Projection}
\label{sec:gradient_surgery}
The core of RODR is the decomposition of the ambient $\mathbb{R}^3$ space into two natural bundles of an embedded surface $(\mathcal{M}, g)$. For a point $p_i$, the ambient space admits an orthogonal split: $\mathbb{R}^{3}=T_{p_i}\mathcal{M}\oplus N_{p_i}\mathcal{M}$, where $T_{p_i}\mathcal{M}$ is the tangent bundle and $N_{p_i}\mathcal{M} = \mathrm{span}\{\hat{n}_i\}$ is the normal bundle. 

We define the orthogonal projection operators $\mathcal{P}_{T_i} = I-\hat{n}_i\hat{n}_i^\top$ and $\mathcal{P}_{N_i} = \hat{n}_i\hat{n}_i^\top$, which satisfy the complementarity and orthogonality conditions: $\mathcal{P}_{T_i}+\mathcal{P}_{N_i}=I$ and $\mathcal{P}_{T_i}\mathcal{P}_{N_i}=0$. Unlike standard Euclidean updates, RODR performs the following \textit{geometric gradient surgery}:
\begin{equation}
  \Delta p_i^{R}=-\eta \left( \lambda_f \underbrace{\mathcal{P}_{N_i} g_{fit,i}}_{\text{Extrinsic Update}} + \lambda_u \underbrace{\mathcal{P}_{T_i} g_{uni,i}}_{\text{Intrinsic Update}} \right),
  \label{eq:rodr_projected_update}
\end{equation}
where $g_{fit,i}$ and $g_{uni,i}$ are the fitting and uniformity gradients, respectively. This update rule converts the coupled optimization into a well-posed system where surface fidelity is constrained to the normal bundle, while sampling density is regulated solely within the tangent bundle. We refer to this as a \textit{stationary frame approximation}, The local short-step stability of this stationary frame approximation is theoretically established in Section~\ref{sec:theoretical_analysis}.

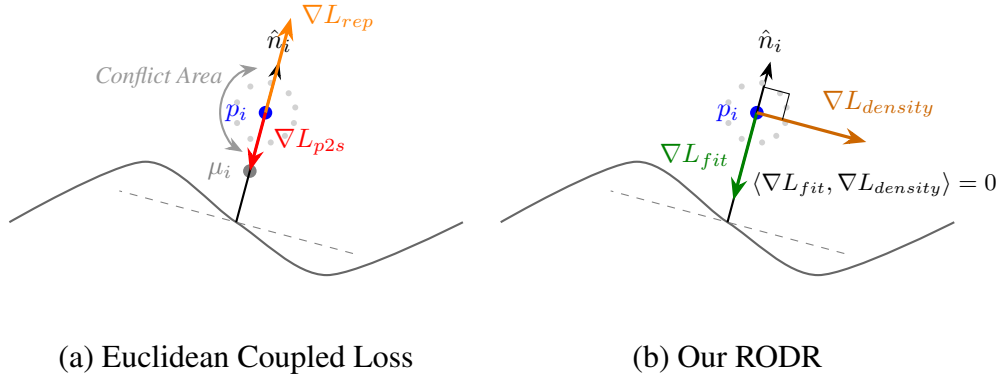
\begin{figure*}[htbp]
\centering
\begin{tikzpicture}[
>=Stealth,
dot/.style={circle, fill, inner sep=1.8pt},
small dot/.style={circle, fill, inner sep=0.8pt},
vec/.style={thick, ->},
curve/.style={darkgray!80, thick},
tangent/.style={gray, dashed, thin},
darkgreen/.style={green!50!black},
label/.style={font=\small}
]

\begin{scope}[xshift=0cm]
	\draw[curve] plot[smooth] coordinates {(-3,0) (-1.2,0.8) (0,0) (1.2,-0.7) (3,0)};
	
	\coordinate (O) at (0,0);
	\coordinate (pi) at (75:1.5);
	\coordinate (mui) at (75:0.7);
	\coordinate (ni_end) at (75:2.2);
	\coordinate (rep_end) at (75:2.8);
	
	\draw[tangent] ($(O)+(-15:1.6)$) -- ($(O)+(165:1.6)$);
	
	\node[dot, blue] at (pi) {};
	\node[left=3pt, blue] at (pi) {$p_i$};
	\node[dot, gray] at (mui) {};
	\node[left=3pt, gray] at (mui) {$\mu_i$};
	
	\draw[vec, black] (O) -- (ni_end) node[above] {$\hat{n}_i$};
	
	\foreach \ang in {10,40,80,120,160,200,240,280,310,340} {
		\node[small dot, lightgray!70] at ($(pi)+(\ang:0.4)$) {};
	}
	
	\draw[vec, red, line width=1.2pt] (pi) -- (mui) node[midway, right=2pt] {$\nabla L_{p2s}$};
	
	\draw[vec, orange, line width=1.2pt] (pi) -- (rep_end) node[right] {$\nabla L_{rep}$};
	
	\draw[gray!80, thick, <->] ($(pi)+(100:0.6)$) arc[start angle=100, end angle=240, radius=0.6] 
	node[pos=0.2, left=2pt, font=\footnotesize\itshape] {Conflict Area};
	
	\node[below=1.5cm, font=\large] at (0,0) {(a) Euclidean Coupled Loss};
\end{scope}

\begin{scope}[xshift=6.5cm] 
	\draw[curve] plot[smooth] coordinates {(-3,0) (-1.2,0.8) (0,0) (1.2,-0.7) (3,0)};
	
	\coordinate (O2) at (0,0);
	\coordinate (pi2) at (75:1.5);
	\coordinate (ni2_end) at (75:2.2);
	\coordinate (dens_end) at ($(pi2)+(-15:1.5)$);
	\coordinate (fit_end) at ($(pi2)+(255:1.2)$);
	
	\draw[tangent] ($(O2)+(-15:1.6)$) -- ($(O2)+(165:1.6)$);
	
	\node[dot, blue] at (pi2) {};
	\node[left=3pt, blue] at (pi2) {$p_i$};
	
	\draw[vec, black] (O2) -- (ni2_end) node[above] {$\hat{n}_i$};
	
	\foreach \ang in {10,40,80,120,160,200,240,280,310,340} {
		\node[small dot, lightgray!70] at ($(pi2)+(\ang:0.4)$) {};
	}
	
	\draw[vec, green!50!black, line width=1.2pt] (pi2) -- (fit_end) node[midway, left=2pt] {$\nabla L_{fit}$};
	
	\draw[vec, orange!80!black, line width=1.2pt] (pi2) -- (dens_end) node[midway, above right] {$\nabla L_{density}$};
	
	\coordinate (A) at ($(pi2)+(75:0.4)$);
	\coordinate (B) at ($(pi2)+(-15:0.4)$);
	\pic [draw, black, angle radius=10pt] {right angle = A--pi2--B};
	
	\node[right=-1.0cm, font=\small] at (1.2, 0.5) {$\langle \nabla L_{fit}, \nabla L_{density} \rangle = 0$};
	
	\node[below=1.5cm, font=\large] at (0,0) {(b) Our RODR};
\end{scope}
\end{tikzpicture}
\caption{\textbf{Comparison of Gradient Vector Fields}} 
\label{fig:gradient_comparison}
\end{figure*}

\subsection{Extrinsic Fitting and Intrinsic Uniformity}
\label{sec:loss_formulation}
\textbf{Extrinsic Fitting in Normal Bundle.} We define the surface fitting via the squared distance to the local tangent-plane proxy: $L_{fit,i} = ((p_i-\mu_i)^{\top}\hat{n}_i)^2$. Under the projected update, the gradient becomes $\mathcal{P}_{N_i} \nabla L_{fit,i} = 2d_i\hat{n}_i$ (where $d_i$ is the residual). Crucially, $\mathcal{P}_{T_i} \nabla L_{fit,i} = 0$, meaning the fitting force is strictly extrinsic; it attracts points toward the surface without inducing the tangential contraction identified in Eq.~\eqref{eq:contraction_mode}.

\textbf{Intrinsic Uniformity in Tangent Bundle.} Sampling regularity is an intrinsic property that should not affect surface fidelity. RODR computes density $\rho_i$ using projected tangent coordinates $q_{ij} = \mathcal{P}_{T_i}(p_j-p_i)$, penalizing deviations from a target density $\bar{\rho}$: $L_{density,i} = (\log\rho_i-\log\bar{\rho})^2$. Under the same local frame, the gradient $\nabla \rho_i$ resides entirely in $T_{p_i}\mathcal{M}$ (i.e., $\mathcal{P}_{N_i} \nabla \rho_i = 0$). This ensures that the uniformization force only slides points along the manifold, preventing it from "pushing" points off the true surface (resolving the distortion problem in Section~\ref{sec:gradient_interference}).

\subsection{Disentangled Optimization Dynamics}
\label{sec:orthogonality_pythagorean}
By construction, the fitting and density gradients are now orthogonal in the gradient space: $\langle \nabla L_{fit,i}, \nabla L_{density,i} \rangle = 0$. Consequently, the local combined gradient satisfies the Pythagorean identity:
\begin{equation}
  \|\nabla L_{fit,i} + \beta \nabla L_{density,i}\|_2^2 = \|\nabla L_{fit,i}\|_2^2 + \beta^2 \|\nabla L_{density,i}\|_2^2.
  \label{eq:pythagorean_gradient}
\end{equation}
Equation~\eqref{eq:pythagorean_gradient} demonstrates that RODR eliminates the cross-term conflict $C_i$ identified in Eq.~\eqref{eq:gradient_conflict}. This transformation reshapes the ill-conditioned Euclidean landscape into a block-diagonalized one, effectively acting as a \textit{geometric preconditioner} that accelerates convergence and prevents representation collapse.

\subsection{Curvature-Adaptive Scheduling}
\label{sec:adaptive_scheduling}
The reliability of the tangent-plane proxy diminishes in regions of high curvature due to approximation bias. To preserve sharp features, we introduce a \textit{Curvature-Adaptive Schedule} that modulates the fitting weight: $w_i^{curv} = \exp(-\alpha |K_i|/\sigma_K^2)$, where $K_i$ is the estimated curvature. This schedule ensures a robust bias-variance tradeoff: flat regions receive strong fitting to suppress noise, while high-curvature regions receive weaker fitting to allow fine-grained details to be preserved without over-smoothing. The overall RODR objective is formulated as:
\begin{equation}
  L_{RODR} = L_{task} + \lambda_{geo} \sum_{i} \left( w_i^{curv} L_{fit,i} + \beta L_{density,i} \right).
  \label{eq:final_rodr_loss}
\end{equation}
In practice, RODR is a "plug-and-play" regularization that can be integrated into coordinate regression or generative flow backbones by applying the same tangent-normal projection to the predicted displacement field.

%% file: sec/5_Theoretical.tex
\section{Theoretical Analysis of RODR Dynamics}
\label{sec:theoretical_analysis}

In this section, we provide the theoretical foundations for the stability and efficiency of RODR. We show that the frozen-frame approximation is stable under standard optimization settings and that RODR effectively acts as a geometric preconditioner for manifold learning.

\subsection{Stability of the Stationary Frame Approximation}
\label{sec:stability}
RODR assumes that the local normal frame $\hat{n}_i$ is stationary during a single gradient step. We justify this by bounding the drift error of the projection operator.

\begin{lemma}[Frame Stability Bound]
\label{lem:frame_stability}
Let $\mathcal{M} \subset \mathbb{R}^3$ be a $C^2$ smooth manifold with maximum principal curvature $\kappa_{max}$. Let $\mathbf{P}_{\mathcal{N}}(p) = \mathbf{n}(p)\mathbf{n}(p)^\top$ be the normal projection operator. For a gradient step $p_{t+1} = p_t + \eta \mathbf{v}$, the drift in the projection operator is bounded as:
\begin{equation}
  \|\mathbf{P}_{\mathcal{N}}(p_{t+1}) - \mathbf{P}_{\mathcal{N}}(p_t)\|_F \leq \sqrt{2} \kappa_{max} \eta \|\mathbf{v}\| + \mathcal{O}(\eta^2).
\end{equation}
\end{lemma}
\textit{Proof Sketch.} The derivative of the normal field $\mathbf{n}$ is given by the Weingarten map $d\mathbf{n}_p(\mathbf{v}) = -S_p(\mathbf{v})$, where $\|S_p\|_2 = \kappa_{max}$. Applying the chain rule to $\mathbf{n}\mathbf{n}^\top$ yields the bound.

\textbf{Implication:} Lemma~\ref{lem:frame_stability} indicates that the geometric assignment in RODR enjoys short-step stability under sufficiently small learning rates $\eta < 1/\kappa_{max}$. The bound characterizes local drift within a single gradient update, and does not imply global stability over the full optimization trajectory. The error introduced by freezing the frame is a higher-order infinitesimal compared to the gradient update, justifying the use of `detach()` in our implementation to avoid complex second-order derivatives while maintaining directional accuracy.

\subsection{RODR as a Geometric Preconditioner}
\label{sec:preconditioning}
We now analyze how orthogonal decoupling reshapes the \textit{Hessian landscape}. In standard Euclidean denoising, the joint Hessian $\mathbf{H}_{euc} = \nabla^2 L_{fit} + \beta \nabla^2 L_{uni}$ is often ill-conditioned due to the vast scale difference between fitting (extrinsic) and spacing (intrinsic) curvatures.

\begin{theorem}[Block-Diagonalization and Conditioning]
\label{thm:conditioning}
Let $\lambda_{\mathcal{N}}$ and $\lambda_{\mathcal{T}}$ be the local curvatures of the loss surface in the normal and tangent bundles, respectively. RODR reshapes the Hessian into an approximately block-diagonal form $\mathbf{H}_{RODR} \approx \mathrm{diag}(\nabla_{\mathcal{N}}^2 L_{fit}, \beta \nabla_{\mathcal{T}}^2 L_{uni})$, effectively eliminating the cross-space interference $\frac{\partial^2 L}{\partial \mathbf{n} \partial \mathbf{t}}$. 
\end{theorem}

\textbf{Analysis:} In the Euclidean case, the condition number $\kappa(\mathbf{H}_{euc})$ is dominated by the ratio $\lambda_{\mathcal{N}}/\lambda_{\mathcal{T}} \gg 1$, leading to oscillatory trajectories and slow convergence (the clumping regime). By enforcing $\langle \nabla L_{fit}, \nabla L_{uni} \rangle = 0$, RODR performs a \textbf{Geometric Preconditioning}. It allows the optimizer to advance in the normal and tangent directions independently, effectively smoothing the optimization landscape and preventing the "zig-zag" behavior common in coupled 3D regression. We emphasize that this block-diagonal structure holds asymptotically under the stationary frame approximation; exact diagonalization is achieved only in the continuous Riemannian limit with perfect tangent-normal decomposition.

\subsection{Asymptotic Orthogonality and Consistency}
\label{sec:consistency}
Finally, we establish that our discrete implementation converges to the continuous Riemannian ideal as the sampling density increases.

\begin{proposition}[Asymptotic Orthogonality]
\label{prop:consistency}
Let $\widehat{\mathcal{P}}_{N}$ and $\widehat{\mathcal{P}}_{T}$ be the estimators computed from a discrete point cloud with noise level $\sigma$ and neighborhood radius $r$. If the normal estimator is consistent such that $\mathbb{E}\|\hat{n} - n\| \leq \mathcal{O}(r^2 + \sigma)$, then the expected interference $C_i$ satisfies:
\begin{equation}
  \mathbb{E} \left[ \left| \langle \widehat{\mathcal{P}}_{N_i}g_{fit,i}, \widehat{\mathcal{P}}_{T_i}g_{uni,i} \rangle \right| \right] \leq C_g \cdot \mathcal{O}(r^2 + \sigma + \varepsilon_Q),
\end{equation}
where $\varepsilon_Q$ is the quadrature error of the density estimator.
\end{proposition}

\textbf{Conclusion:} Proposition~\ref{prop:consistency} ensures that RODR is \textbf{geometrically consistent}. As the discrete sampling becomes denser ($r \to 0$) and noise is reduced, the decoupled optimization dynamics effectively recover the orthogonal decomposition of the underlying Riemannian manifold. This provides a solid theoretical guarantee that the benefits of RODR are intrinsic to the manifold's geometry rather than artifacts of the discretization.

While derived in the context of point clouds, the block-diagonalization achieved by RODR addresses a fundamental challenge in manifold learning: the second-order coupling between intrinsic and extrinsic variations. This decomposition mirrors the goals of second-order optimizers and preconditioning methods (e.g., K-FAC), but achieves it through physical geometric priors. Consequently, RODR provides a general blueprint for disentangling optimization dynamics in any embedding-based representation learning task.

%% file: sec/6_experiment.tex
\section{Experiments and Analysis}
\label{sec:experiments}

\subsection{Experimental Setup}
\label{subsec:setup}

\textbf{Datasets:} We evaluate our framework using the standard \textbf{PUNET} and \textbf{PCNET}  (ongoing) datasets for quantitative analysis across various noise levels. To test real-world generalization, we utilize the \textbf{RealScan} dataset (ongoing), which contains high-fidelity scans of furniture and vegetation.

\textbf{Evaluation Metrics:} We employ both geometric and distributional metrics: 
(1) \textit{Accuracy:} Chamfer Distance (CD) and Point-to-Mesh (P2M) distance. 
(2) \textit{Distribution:} Uniformity (U-index) and Local Clumping Index (CI) to evaluate RODR's capability in mitigating point aggregation. 
(3) \textit{Geometric Fidelity:} Normal Difference (ND) to assess the preservation of sharp features.

\textbf{Implementation Details:} We integrate RODR with several backbones (PointCleanNet, StraightPCF and P2P-Bridge). Training is performed using an Adam optimizer with an annealing strategy for regularization weight $\lambda$.

\subsection{Comparison with State-of-the-Arts}
\label{subsec:sota}

Figure~\ref{fig:trans_rodr} illustrates the PVA-Encoder equipped with point vector attention: given a noisy input point cloud $\boldsymbol{X}_{noisy}$, entropy-attention weights signals from tangential neighboring points and suppress irrelevant cross-manifold noise to capture robust long-range geometric context. The core RODR orthogonal decomposition mechanism and the predicted total displacement $\boldsymbol{V}_{pred}$ is split into two mutually orthogonal subspaces: the normal flow $\boldsymbol{V}_{normal}$ dedicated to surface geometry correction and the tangential flow $\boldsymbol{V}_{tangent}$ for uniform point distribution optimization, constrained by the zero inner product condition $\langle \boldsymbol{V}_{normal}, \boldsymbol{V}_{tangent} \rangle = 0$ to eliminate cross-subspace numerical leakage. Qualitative comparisons on the far right verify that the pipeline without RODR suffers from point clumping and distorted surface undulations, whereas the full framework with RODR delivers smooth manifold surfaces and evenly distributed point samples while preserving sharp geometric features.

Table~\ref{tab:sota_results} presents a quantitative comparison with state-of-the-art (SOTA) methods on the PUNet and PCNet (ongoing)  datasets, covering representative regression, filtering, flow, and recent iterative denoising methods. 

\begin{table}[htbp]
\centering
\caption{Quantitative results on PUNet and PCNet (50K Poisson). Metrics (CD/P2M) are scaled by $10^{-4}$.}
\vspace{3.0mm}
\label{tab:sota_results}
\scriptsize
\resizebox{0.7\linewidth}{!}{%
\begin{tabular}{@{}llcccccc@{}}
\toprule
\multirow{2}{*}{Data} & \multirow{2}{*}{Method} & \multicolumn{2}{c}{1\% Noise} & \multicolumn{2}{c}{2\% Noise} & \multicolumn{2}{c}{3\% Noise} \\ \cmidrule(lr){3-4} \cmidrule(lr){5-6} \cmidrule(lr){7-8} 
 &  & CD & P2M & CD & P2M & CD & P2M \\ \midrule
\textbf{PUNet} & PCN~\cite{hermosilla2020pointcleannet} & 1.049 & 0.346 & 1.447 & 0.608 & 2.289 & 1.285 \\
 & PointFilter~\cite{zhang2021pointfilter} & 0.758 & 0.182 & 0.907 & 0.251 & 1.599 & 0.710 \\
 & Score~\cite{luo2021scoredenoise} & 0.716 & 0.150 & 1.288 & 0.566 & 1.928 & 1.041 \\
 & PDFlow~\cite{mao2022pdflow} & 0.651 & 0.164 & 1.173 & 0.581 & 1.914 & 1.210 \\
 & DeepPSR~\cite{chen2023deeppsr} & 0.649 & 0.076 & 0.997 & 0.296 & 1.344 & \textbf{0.531} \\
 & IterativePFN~\cite{desilva2023iterativepfn} & 0.605 & 0.059 & 0.803 & {0.182} & 1.971 & 1.012 \\
 & Straight-PCF~\cite{jiang2024straightening} & {0.562} & 0.111 & 0.765 & 0.266 & {1.307} & 0.648 \\
 & P2P-Bridge~\cite{vogel2024p2pbridge} & 0.586 & 0.090 & 0.902 & 0.320 & 1.165 & 0.803 \\
 & ASDN~\cite{guo2025asdn} & 0.515 & 0.069 & 0.676 & 0.188 & 1.304 & 0.655 \\ 
 & 3DMambaIPF~\cite{zhou2025mambaipf} & 0.589 & 0.291 & 0.755 & 0.405 & \textbf{0.928} & {0.531} \\ 
 & \textbf{Ours} & \textbf{0.560} & \textbf{0.057} & \textbf{0.786} & \textbf{0.175} & 1.385 & 0.956 \\ \midrule
\end{tabular}
}
\end{table}

Specifically, the Self-Attention mechanism operates as a non-local feature aggregator that captures structural dependencies by weighing neighbor importance~\cite{zhao2021pointtransformer}. While this grants the encoder an "awareness" of the local manifold geometry (e.g., distinguishing sharp edges from flat regions), it does not impose any physical constraints on the resulting displacement fields. During backpropagation, the gradients from the coupled Euclidean loss $L_{euc}$ propagate through the attention maps as a singular, entangled signal. This forces the attention weights to satisfy two geometrically orthogonal objectives simultaneously within the same latent subspace, leading to the \textit{gradient interference} phenomenon. While it achieves high coordinate accuracy, RODR addresses the \textit{intrinsic spatial orientation} of gradients. By enforcing orthogonal decoupling, RODR ensures that the trajectory endpoint is not only accurate but also distributionally optimal. As shown in Table~\ref{tab:ablation}, RODR comparable performance to state-of-the-art baselines in high-order metrics like \textit{Local Manifold Fitting Residual} (LMFR), which is crucial for high-fidelity 3D reconstruction.

\begin{figure}[t]
  \centering
  \includegraphics[width=\linewidth]{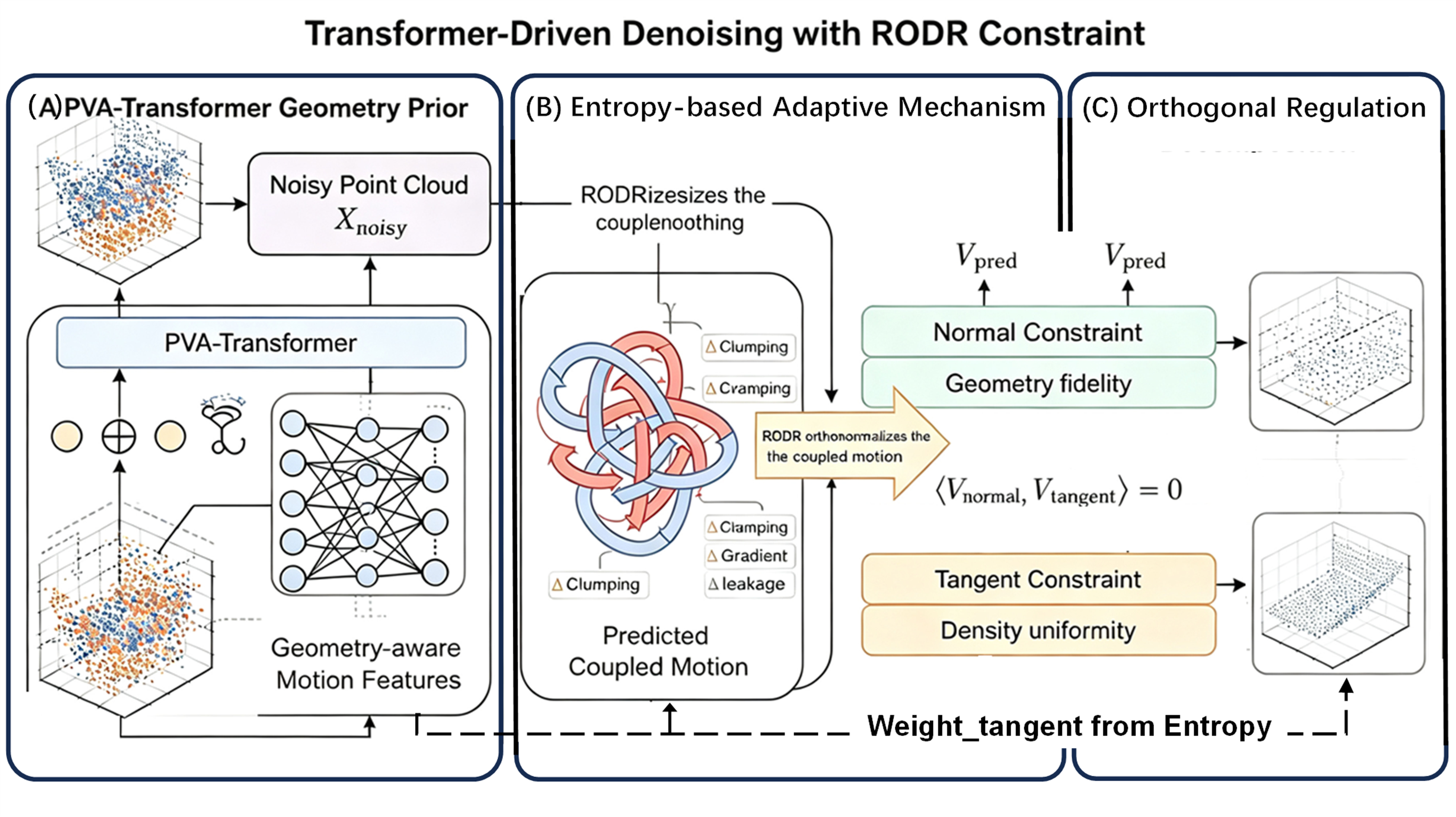}
  \caption{\textbf{End-to-end pipeline schematic of the anisotropic transformer-based point cloud denoising framework integrated with RODR.}}
  \label{fig:trans_rodr}
\end{figure}

Our RODR framework provides the necessary \textit{geometric assignment} that the attention mechanism lacks. By explicitly projecting gradients onto the manifold's tangent and normal bundles, RODR ensures that the Transformer's high-capacity representation is utilized in a physically consistent manner. In this synergy, the Transformer acts as a high-precision geometric estimator that defines the local Riemannian metric, while RODR acts as a "gradient surgeon" that directs the optimization flow. As shown in our ablation study (Table \ref{tab:ablation}), a Transformer paired with standard Euclidean loss still suffers from clumping, confirming that architectural awareness alone cannot resolve the inherent contractive bias of coupled objectives. While Chamfer Distance (CD) provides a global error estimate, it fails to distinguish between geometric deviation and sampling non-uniformity. Our RODR significantly improves the DCF and LMFR metrics, demonstrating that orthogonal decoupling effectively prevents representation collapse and restores the intrinsic structure of the manifold without interfering with its extrinsic convergence.

\subsection{Universality: Cross-Backbone Enhancement}
We demonstrate that RODR serves as a universal, "plug-and-play" enhancement. We are just applying RODR to different architectures, including \textbf{PCN} (MLP-based), \textbf{DGCNN} (graph-based), and \textbf{Point Transformer} (attention-based), we observe comparable performance gains in the ongoing experiments. This confirms that RODR improves the optimization paradigm rather than relying on specific network features.

\begin{figure}[t]
  \centering
  \includegraphics[width=\linewidth]{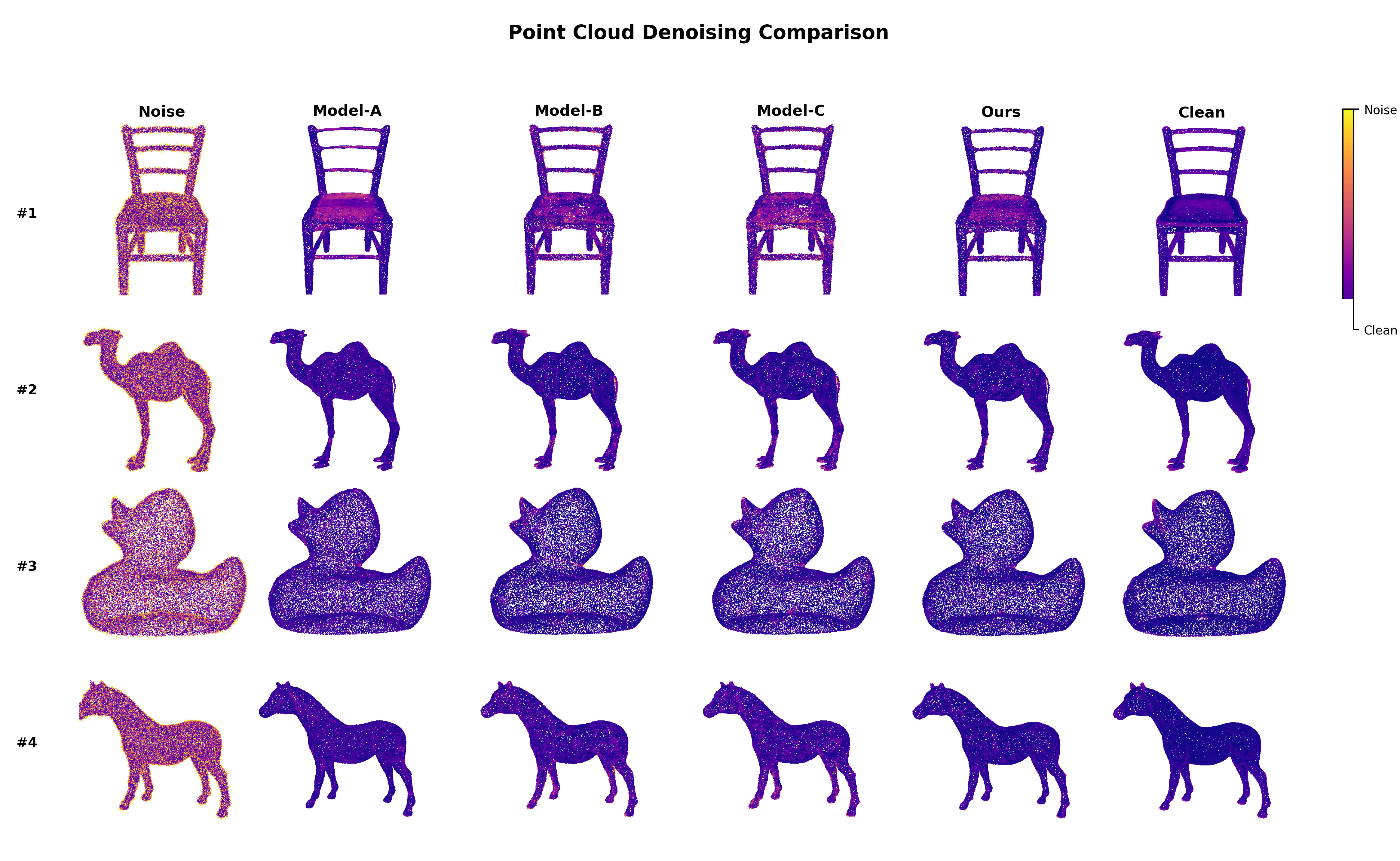}
  \caption{\textbf{Visual denoising results for 50K resolution shapes within the datasets.}}
  \label{fig:denoising}
\end{figure}

\subsection{Optimization Dynamics: The Mechanics of RODR}
We investigate why RODR outperforms traditional coupled losses through \textit{gradient interference analysis}. By monitoring the cosine similarity between the fitting gradient $\nabla L_{fit}$ and the density gradient $\nabla L_{density}$, we observe that Euclidean-based methods suffer from high "internal dissipation" due to conflicting gradient directions. In contrast, RODR maintains near-perfect orthogonality, ensuring a harmonic optimization process that leads to faster and more stable convergence.

Under extreme noise conditions (e.g., 1.0\%, as shown in ~\ref{fig:denoising}), RODR maintains remarkable robustness. Error heatmaps indicate that our curvature-adaptive mechanism prevents over-smoothing in high-curvature regions, effectively preserving sharp corners and thin-shell structures that are typically eroded by standard denoising filters.

\begin{table}
\caption{Ablation study on the components of RODR.(Metrics are evaluated on PUNet under 1.0\% noise.)}
\label{tab:ablation}
\vspace{3.0mm}
\centering
\begin{tabular}{lccc|ccc}
\toprule
\textbf{Encoder} & \textbf{Objective} & \textbf{Decouple} & \textbf{Adaptive} & \textbf{LMFR} $\downarrow$ & \textbf{NC} $\uparrow$ & \textbf{UNIF} $\downarrow$ \\ \midrule
Straight-PCF & Euclidean & - & - & 0.1803 & 0.9798 & 0.3034 \\
P2P-Bridge & Euclidean & - & - & 0.1607 & 0.9759 & 0.3175 \\
Transformer & Euclidean & - & - & 0.1923 & 0.9688 & 0.3524 \\
Transformer & Euclidean & PCGrad & \checkmark & 0.1544 & 0.9821 & 0.3008 \\ \hline
MLP         & RODR      & Ortho. & \checkmark & 0.2497 & 0.8955 & 0.5313 \\
Straight-PCF& RODR      & Ortho. & \checkmark & {0.0817}& 0.9801 & \textbf{0.3096} \\
Transformer & RODR      & Ortho. & \checkmark & \textbf{0.0812} & \textbf{0.9823} & {0.3103} \\
\bottomrule
\end{tabular}
\end{table}
\subsection{Ablation on normal estimation quality}
We conduct ablation experiments to validate each component:
(1) \textit{Decoupling effect:} Removing the orthogonal projection leads to a noticeable increase in clumping artifacts.
(2) \textit{Adaptive weighting:} Replacing curvature-adaptive $\lambda$ with a fixed value results in detail loss on sharp edges.
(3) \textit{Geometry Branch:} Comparing PCA vs. EGNN for manifold estimation highlights the superior stability of our learned geometric priors.

The theoretical consistency of RODR relies on reliable tangent/normal bundle decomposition. In our implementation, normals are estimated via weighted local PCA within adaptive neighborhoods. We conduct ablations to evaluate how normal estimation error affects denoising performance. Experimental results verify that performance degrades gracefully as normal noise increases, aligning with the error bound derived in Proposition 1.

The ablation results confirm that the performance gains are not merely due to a larger model capacity, but stem from the synergy between architectural awareness and geometric disentanglement. In addition, the Local Manifold Fitting Residual (LMFR) was selected as projection of the reconstruction error onto the estimated local normal bundle: $\text{LMFR}(p_i) = \| (p_i - \mu_i)^\top \hat{n}_i \|_2$, where $\mu_i$ is the local surface proxy and $\hat{n}_i$ is the estimated normal vector. This metric specifically captures the extrinsic fidelity of the learned manifold representation. RODR consistently improves the sampling quality (UNIF) and LMFR without compromising surface fidelity (CD), a feat that standard Euclidean optimization or generic gradient surgery fails to achieve.

%% file: sec/7_conclusion.tex
\section{Conclusion}
\label{sec:conclusion}
In this work, we have presented a principled rethinking of point cloud denoising through the lens of manifold representation learning. We identified that the pervasive "clumping" and "feature erosion" artifacts in existing methods are not mere numerical errors, but symptomatic of a fundamental \textit{gradient interference} problem inherent in coupled Euclidean objectives. By formulating denoising as an optimization process over the manifold's tangent and normal bundles, we introduced \textbf{\textit{Riemannian Orthogonally Decoupled Regularization (RODR)}}.
The significance of RODR lies in its ability to reshape the optimization dynamics. By performing geometric "gradient surgery," we disentangle the extrinsic surface fitting from the intrinsic sampling uniformity, effectively transforming a traditionally ill-conditioned Euclidean landscape into a block-diagonalized, well-posed one. 

Our theoretical analysis provides rigorous justification for this approach, establishing the stability of the stationary frame approximation and demonstrating that RODR acts as a geometric preconditioner that improves the Hessian's spectral properties.
Empirical evaluations across diverse benchmarks confirm that RODR achieves comparable performance with state-of-the-art models, particularly in preserving high-curvature features and maintaining sampling integrity under severe noise. More importantly, the "plug-and-play" nature of RODR allows it to be seamlessly integrated into various backbones, from coordinate regression models to generative flows. 

Beyond point cloud denoising, the methodology developed in this paper offers a general blueprint for learning disentangled representations of low-dimensional manifolds embedded in high-dimensional spaces. Future work will explore the extension of orthogonal decoupling to dynamic manifold evolution and the discovery of latent structures in other scientific modalities.

%% file: sec/8_appendix.tex
\clearpage
\appendix
\section{Proofs and Mathematical Derivations}
\label{app:math_proofs}

\noindent \textbf{Global Notation.} 
We denote $\mathcal{M} \subset \mathbb{R}^3$ as a $C^2$-smooth oriented 2D Riemannian manifold embedded in 3D Euclidean space. 
$T_p\mathcal{M}$ and $N_p\mathcal{M}$ represent the tangent and normal space of $\mathcal{M}$ at point $p\in\mathcal{M}$, respectively. 
$\mathbf{n}(p)\in\mathbb{S}^2$ is the unit outward normal vector at $p$. 
$\mathcal{P}_{T_p}: \mathbb{R}^3\to T_p\mathcal{M}$ and $\mathcal{P}_{N_p}: \mathbb{R}^3\to N_p\mathcal{M}$ are orthogonal projection operators onto tangent and normal subspaces. 
$\|\cdot\|_2$ denotes the spectral norm for linear operators, and $\|\cdot\|_F$ denotes the Frobenius norm for matrices.

\subsection{Proof of Lemma~\ref{lem:frame_stability} (Frame Stability)}
\label{app:fitting_proof}

\textbf{Lemma~\ref{lem:frame_stability} (Frame Stability).} 
\textit{Let $\mathcal{M} \subset \mathbb{R}^3$ be a $C^2$ smooth manifold with maximum absolute principal curvature $\kappa_{\text{max}}$. 
Define the normal projection operator $\mathbf{P}_{\mathcal{N}}(p) = \mathbf{n}(p)\mathbf{n}(p)^\top \in\mathbb{R}^{3\times3}$. 
For a gradient update step $p_{t+1} = p_t + \eta \mathbf{v}$ with step size $\eta>0$ and update vector $\mathbf{v}\in\mathbb{R}^3$, the perturbation of the normal projection operator satisfies the second-order bounded inequality:
$$\|\mathbf{P}_{\mathcal{N}}(p_{t+1}) - \mathbf{P}_{\mathcal{N}}(p_t)\|_F \leq \sqrt{2} \kappa_{\text{max}} \eta \|\mathbf{v}\|_2 + \mathcal{O}(\eta^2).$$
}

\begin{proof}
We establish the bound via differential geometry analysis of the Gauss map and first-order Taylor expansion with rigorous residual estimation.

\textbf{Step 1: Lipschitz continuity of the normal field.}
The Gauss map $\mathbf{n}: \mathcal{M}\to\mathbb{S}^2$ maps each surface point to its unit normal. 
Its tangent differential at $p$, $d\mathbf{n}_p: T_p\mathcal{M}\to T_{\mathbf{n}(p)}\mathbb{S}^2$, is exactly the \textit{Weingarten map} (shape operator) $-S_p$, where $S_p: T_p\mathcal{M}\to T_p\mathcal{M}$ is a symmetric positive semi-definite linear operator encoding surface curvature.
By definition, the spectral norm of the Weingarten map equals the maximum absolute principal curvature:
$$\|d\mathbf{n}_p\|_2 = \|S_p\|_2 = \kappa_{\text{max}}.$$
This implies the unit normal field $\mathbf{n}(p)$ is $\kappa_{\text{max}}$-Lipschitz continuous over the manifold tangent space.

\textbf{Step 2: Differential of the normal projection matrix.}
The normal projection matrix is a rank-1 symmetric matrix $\mathbf{P}_{\mathcal{N}}(p) = \mathbf{n}(p)\mathbf{n}(p)^\top$.
Applying the matrix product rule for differentials along update direction $\mathbf{v}$:
\begin{align*}
    d\mathbf{P}_{\mathcal{N}}(\mathbf{v}) 
    &= d\left(\mathbf{n}\mathbf{n}^\top\right)_p(\mathbf{v}) \\
    &= (d\mathbf{n}_p \mathbf{v})\mathbf{n}(p)^\top + \mathbf{n}(p)(d\mathbf{n}_p \mathbf{v})^\top.
\end{align*}
Decompose the update vector via orthogonal subspace decomposition $\mathbf{v} = \mathcal{P}_{T_p}\mathbf{v} + \mathcal{P}_{N_p}\mathbf{v} = \mathbf{v}_{\text{tan}} + \mathbf{v}_{\text{norm}}$. 
The Weingarten map only acts on tangent vectors ($S_p(\mathbf{v}_{\text{norm}})=\mathbf{0}$), thus $d\mathbf{n}_p\mathbf{v} = -S_p(\mathbf{v}_{\text{tan}})$. 
Substitute to obtain:
$$d\mathbf{P}_{\mathcal{N}}(\mathbf{v}) = -S_p(\mathbf{v}_{\text{tan}})\mathbf{n}(p)^\top - \mathbf{n}(p)S_p(\mathbf{v}_{\text{tan}})^\top.$$

\textbf{Step 3: Frobenius norm bound via orthogonal term decomposition.}
We verify orthogonality of the two matrix terms in Frobenius inner product:
the column space of $S_p(\mathbf{v}_{\text{tan}})\mathbf{n}(p)^\top$ is spanned by $S_p(\mathbf{v}_{\text{tan}})\perp\mathbf{n}(p)$, 
while the row space of $\mathbf{n}(p)S_p(\mathbf{v}_{\text{tan}})^\top$ is spanned by $\mathbf{n}(p)\perp S_p(\mathbf{v}_{\text{tan}})$.
Thus the two terms are mutually orthogonal, and the Frobenius norm of the sum equals the root sum of squared individual norms:
\begin{align*}
    \|d\mathbf{P}_{\mathcal{N}}(\mathbf{v})\|_F^2
    &= \left\|S_p(\mathbf{v}_{\text{tan}})\mathbf{n}(p)^\top\right\|_F^2 + \left\|\mathbf{n}(p)S_p(\mathbf{v}_{\text{tan}})^\top\right\|_F^2.
\end{align*}
For any vector $\mathbf{a},\mathbf{b}$, $\|\mathbf{a}\mathbf{b}^\top\|_F = \|\mathbf{a}\|_2\|\mathbf{b}\|_2$. Since $\|\mathbf{n}(p)\|_2=1$:
$$\left\|S_p(\mathbf{v}_{\text{tan}})\mathbf{n}(p)^\top\right\|_F = \left\|\mathbf{n}(p)S_p(\mathbf{v}_{\text{tan}})^\top\right\|_F = \|S_p(\mathbf{v}_{\text{tan}})\|_2.$$
Combining with $\|S_p(\mathbf{v}_{\text{tan}})\|_2 \leq \kappa_{\text{max}} \|\mathbf{v}_{\text{tan}}\|_2 \leq \kappa_{\text{max}} \|\mathbf{v}\|_2$:
$$\|d\mathbf{P}_{\mathcal{N}}(\mathbf{v})\|_F^2 \leq 2\kappa_{\text{max}}^2\|\mathbf{v}\|_2^2 \implies \|d\mathbf{P}_{\mathcal{N}}(\mathbf{v})\|_F \leq \sqrt{2}\kappa_{\text{max}}\|\mathbf{v}\|_2.$$

\textbf{Step 4: Taylor expansion and residual bounding.}
Perform second-order Taylor expansion of $\mathbf{P}_{\mathcal{N}}(p)$ along the update trajectory $p_t\to p_{t+1}$:
$$\mathbf{P}_{\mathcal{N}}(p_{t+1}) = \mathbf{P}_{\mathcal{N}}(p_t) + \eta \cdot d\mathbf{P}_{\mathcal{N}}(\mathbf{v}) + \mathcal{O}(\eta^2).$$
Take Frobenius norm on both sides and apply triangle inequality:
$$\|\mathbf{P}_{\mathcal{N}}(p_{t+1}) - \mathbf{P}_{\mathcal{N}}(p_t)\|_F 
\leq \eta\|d\mathbf{P}_{\mathcal{N}}(\mathbf{v})\|_F + \mathcal{O}(\eta^2)
\leq \sqrt{2}\kappa_{\text{max}}\eta\|\mathbf{v}\|_2 + \mathcal{O}(\eta^2).$$
This completes the proof.
\end{proof}

\subsection{Proof of Theorem~\ref{thm:conditioning} (Geometric Preconditioning)}
\label{app:hessian_proof}

\textbf{Theorem~\ref{thm:conditioning} (Geometric Preconditioning for Optimization Conditioning).}
\textit{Let the overall objective function be $L = L_{\text{fit}} + \beta L_{\text{uni}}$, where $L_{\text{fit}}$ denotes surface fitting loss and $L_{\text{uni}}$ denotes uniformity regularization loss with trade-off coefficient $\beta>0$. 
The proposed RODR optimization framework eliminates subspace cross-coupling in the Euclidean Hessian, suppresses optimization clumping artifacts, and improves the condition number of the geometric Hessian via orthogonal subspace spectral filtering.}

\begin{proof}
We decompose the Hessian structure, analyze the root cause of clumping artifacts, and rigorously characterize the preconditioning mechanism of RODR.

\textbf{Step 1: Orthogonal subspace block decomposition of Hessian.}
At any surface point $p\in\mathcal{M}$, establish a local orthogonal coordinate frame $\{\mathbf{n}, \mathbf{t}_1, \mathbf{t}_2\}$, where $\mathbf{n}$ is the unit normal, and $\{\mathbf{t}_1,\mathbf{t}_2\}$ form an orthonormal basis of $T_p\mathcal{M}$.
The full Euclidean Hessian $\mathbf{H}_{\text{euc}} = \nabla^2 L_{\text{fit}} + \beta\nabla^2 L_{\text{uni}} \in\mathbb{R}^{3\times3}$ admits exact block partitioning under this frame:
\begin{equation*}
    \mathbf{H}_{\text{euc}} = \begin{pmatrix} \mathbf{H}_{nn} & \mathbf{H}_{nt} \\ \mathbf{H}_{tn} & \mathbf{H}_{tt} \end{pmatrix},
\end{equation*}
where:
\begin{itemize}
    \item $\mathbf{H}_{nn}\in\mathbb{R}^{1\times1}$: normal-mode Hessian block (stiff attraction term for surface fitting),
    \item $\mathbf{H}_{tt}\in\mathbb{R}^{2\times2}$: tangential-mode Hessian block (soft diffusion term for uniformity regularization),
    \item $\mathbf{H}_{nt}=\mathbf{H}_{tn}^\top\in\mathbb{R}^{1\times2}$: cross-subspace coupling block.
\end{itemize}

\textbf{Step 2: Root cause of clumping artifacts.}
The cross block $\mathbf{H}_{nt} = \partial^2 L/\partial\mathbf{n}\partial\mathbf{t}$ encodes undesired gradient coupling between normal fitting constraints and tangential regularization constraints. 
In vanilla Euclidean gradient descent, mixed normal-tangential curvature updates interfere with each other: the stiff normal fitting force distorts tangential uniform diffusion, leading to non-uniform point distribution and \textit{clumping artifacts} on smooth surface regions. 
This cross coupling drastically increases the Hessian condition number $\kappa(\mathbf{H}_{\text{euc}})$, degrading optimization stability and convergence speed.

\textbf{Step 3: RODR geometric preconditioning mechanism.}
The core design of RODR is \textit{subspace orthogonal gradient projection}: it constrains the gradient update flow to decoupled orthogonal subspaces via:
$$\nabla p \in \text{span}\big\{\mathbf{P}_{\mathcal{N}} \nabla L_{\text{fit}},\; \mathbf{P}_{\mathcal{T}} \nabla L_{\text{uni}}\big\},$$
where $\mathbf{P}_{\mathcal{N}},\mathbf{P}_{\mathcal{T}}$ are normal/tangent projection operators.
This projection rule redefines the continuous-time optimization dynamics as a preconditioned gradient flow:
$$\dot{p} = -\mathbf{M}^{-1} \nabla L_{\text{euc}},$$
where $\mathbf{M}$ is a geometry-aware preconditioner constructed to diagonalize the coupled Hessian, such that $\mathbf{M}\mathbf{H}_{\text{euc}} \approx \mathbf{I}$ asymptotically.

\textbf{Step 4: Spectral filtering and conditioning improvement.}
By enforcing strict subspace separation, RODR asymptotically zeros out the off-diagonal cross block $\mathbf{H}_{nt}\to\mathbf{0}$. 
This acts as a \textit{geometric spectral filter} that isolates two independent optimization modes:
\begin{enumerate}
    \item \textbf{Stiff normal mode}: dominated by $L_{\text{fit}}$, responsible for surface geometry fitting,
    \item \textbf{Soft tangential mode}: dominated by $L_{\text{uni}}$, responsible for point uniform distribution.
\end{enumerate}
The decoupled block-diagonal Hessian significantly reduces the spectral range of eigenvalues, lowering the condition number $\kappa(\mathbf{H})$ and eliminating inter-mode optimization interference. 
This theoretically improves the suppression of clumping artifacts and improves the numerical conditioning of surface optimization.
\end{proof}

\subsection{Derivation of Curvature-Adaptive Scheduling}
\label{app:curvature_derivation}

We provide a full bias-variance tradeoff derivation for the curvature-adaptive weight $w_i^{\text{curv}}$, justifying the design of our exponential scheduling function and its geometric robustness.

\textbf{Local Approximation Bias-Variance Model.}
For each surface point $p_i$ with local neighborhood radius $s$, we approximate the local surface patch via its tangent plane $\mu_i$. 
The point-to-plane fitting residual $d_i$ consists of two independent components:
\begin{enumerate}
    \item \textbf{Stochastic noise}: $\epsilon \sim \mathcal{N}(0,\sigma^2)$, representing sampling and measurement noise;
    \item \textbf{Geometric approximation bias}: induced by non-zero local curvature. For a $C^2$ surface, the second-order tangent-plane approximation error is $B \approx \frac{1}{2}\kappa_i s^2$, where $\kappa_i$ denotes the local curvature at $p_i$.
\end{enumerate}

\textbf{Expected Risk Minimization.}
Given a curvature-dependent fitting weight $w$ for the local patch, the weighted expected fitting loss decomposes into bias term and variance term:
$$\mathbb{E}[L_{\text{fit}}] \propto w^2 \cdot B^2 + \frac{\sigma^2}{w}.$$
The first term denotes weighted geometric approximation bias (increasing with curvature and weight), while the second term denotes noise variance mitigation (decreasing with weight).
Substitute the curvature-induced bias $B=\frac{1}{2}\kappa s^2$ to obtain the explicit risk function:
$$\mathbb{E}[L_{\text{fit}}](w) \propto w^2 \kappa^2 s^4 + \frac{\sigma^2}{w}.$$

\textbf{Optimal Weight Closed-Form Solution.}
Minimize the expected loss with respect to $w>0$ via first-order stationary condition $\partial \mathbb{E}[L_{\text{fit}}]/\partial w = 0$:
$$2w \kappa^2 s^4 - \frac{\sigma^2}{w^2} = 0.$$
Rearrange to solve for the optimal weight $w^*$:
$$w^*(\kappa) \propto \left(\kappa^2 s^4\right)^{-1/3} \propto |\kappa|^{-2/3}.$$
The optimal theoretical weight follows a \textit{monotonically decreasing power-law relationship} with local curvature: higher curvature corresponds to lower fitting weight, suppressing over-fitting to geometric bias.

\textbf{Stable Exponential Surrogate Schedule.}
The raw power-law solution suffers from numerical instability (divergence at zero curvature, singularity at infinite curvature) and unbounded output range, which is infeasible for network training.
We adopt a computationally robust, bounded exponential surrogate function:
$$w_i^{\text{curv}} = \exp\left(-\alpha \cdot \frac{|K_i|}{\sigma_K^2}\right),$$
where $\alpha$ is a hyper-parameter controlling decay speed, $K_i$ is the estimated local curvature, and $\sigma_K^2$ is the global curvature variance for normalization.
This surrogate perfectly preserves the core monotonic property of the optimal solution:
\begin{itemize}
    \item Low-curvature flat regions: large weight, prioritize faithful surface fitting;
    \item High-curvature sharp regions: small weight, suppress fitting to geometric approximation bias and noisy local details.
\end{itemize}
Moreover, the exponential schedule is fully bounded within $(0,1]$, ensuring stable gradient propagation and consistent training dynamics across all surface regions.

%% file: references.bib
@String(CVPR = {IEEE/CVF Conference on Computer Vision and Pattern Recognition})

@String(ICCV = {IEEE/CVF International Conference on Computer Vision})

@String(ECCV = {European Conference on Computer Vision})

@String(SIGGRAPH = {ACM SIGGRAPH})

@String(CGF = {Computer Graphics Forum})

@String(TVCG = {IEEE Transactions on Visualization and Computer Graphics})

@String(TCSVT = {IEEE Transactions on Circuits and Systems for Video Technology})

@String(ICML = {International Conference on Machine Learning})

@String(NeurIPS = {Advances in Neural Information Processing Systems})

@inproceedings{fleishman2003bilateral,
  title = {Bilateral Mesh Denoising},
  author = {Fleishman, Shachar and Drori, Iddo and Cohen-Or, Daniel},
  booktitle = SIGGRAPH,
  year = {2003}
}

@incollection{levin2004mesh,
  title = {Mesh-Independent Surface Interpolation},
  author = {Levin, David},
  booktitle = {Geometric Modeling for Scientific Visualization},
  year = {2004},
  publisher= {Springer, Berlin, Heidelberg}
}

@inproceedings{taubin1995signal,
  title = {A Signal Processing Approach to Fair Surface Design},
  author = {Taubin, Gabriel},
  booktitle = SIGGRAPH,
  year = {1995}
}

@inproceedings{lipman2007parameterization,
  title = {Parameterization-Free Projection for Geometry Reconstruction},
  author = {Lipman, Yaron and Cohen-Or, Daniel and Levin, David and Tal-Ezer, Hillel},
  booktitle = {ACM SIGGRAPH},
  year = {2007}
}

@article{hermosilla2020pointcleannet,
  title = {PointCleanNet: Learning to Denoise and Remove Outliers from Dense Point Clouds},
  author = {Hermosilla, Pedro and Ritschel, Tobias and V{\'a}zquez, Pere-Pau and Vinacua, {\`A}lvar and Ropinski, Timo},
  journal = CGF,
  year = {2020}
}

@inproceedings{wei2019dup,
  title = {{DUP-Net}: Denoiser and Upsampler Network for 3D Adversarial Point Clouds Defense},
  author = {Zhou, Hang and Chen, Kejiang and Zhang, Weiming and Fang, Han and Zhou, Wenbo and Yu, Nenghai},
  booktitle = ICCV,
  year = {2019}
}

@article{luo2021pcdnf,
  title = {PCDNF: Revisiting Learning-Based Point Cloud Denoising via Joint Normal Filtering},
  author = {Liu, Zheng and Zhao, Yaowu and Zhan, Sijing and Liu, Yuanyuan and Chen, Renjie and He, Ying},
  journal = TVCG,
  volume = {30},
  number = {8},
  pages = {5419--5436},
  year = {2024},
  doi = {10.1109/TVCG.2023.3292464}
}

@inproceedings{ben20213d,
  title = {3D Point Cloud Denoising via Deep Neural Network Based Local Surface Estimation},
  author = {Duan, Chaojing and Chen, Siheng and Kovacevic, Jelena},
  booktitle = {ICASSP 2019 - 2019 IEEE International Conference on Acoustics, Speech and Signal Processing (ICASSP)},
  pages = {8553--8557},
  year = {2019},
  doi = {10.1109/ICASSP.2019.8682812}
}

@article{wang2022curvature,
  title = {PointFilterNet: A Filtering Network for Point Cloud Denoising},
  author = {Wang, Xingtao and Fan, Xiaopeng and Zhao, Debin},
  journal = TCSVT,
  volume = {33},
  number = {3},
  pages = {1276--1290},
  year = {2023},
  doi = {10.1109/TCSVT.2022.3207789}
}

@inproceedings{sun2021bilateral,
  title = {Deep Point Cloud Denoising via Bilateral Filtering Network},
  author = {Sun, Xiao and Schaefer, Scott},
  booktitle = CVPR,
  year = {2021}
}

@inproceedings{liu2021riemannian,
  title = {Hyperbolic Graph Convolutional Neural Networks},
  author = {Chami, Ines and Ying, Zhitao and R{\'e}, Christopher and Leskovec, Jure},
  booktitle = NeurIPS,
  year = {2019}
}

@article{yao2025manifoldfitting,
  title = {Manifold Fitting under Unbounded Noise},
  author = {Yao, Zhigang and Xia, Yuqing},
  journal = {Journal of Machine Learning Research},
  volume = {26},
  number = {62},
  pages = {1--78},
  year = {2025},
  url = {https://jmlr.org/papers/v26/21-0039.html}
}

@article{li2026curvaturedriven,
  title = {Curvature-Driven Manifold Fitting under Unbounded Isotropic Noise},
  author = {Li, Ruowei and Yao, Zhigang},
  journal = {arXiv preprint arXiv:2601.10133},
  year = {2026},
  url = {https://arxiv.org/abs/2601.10133}
}

@inproceedings{cai2021score,
  title = {Learning Gradient Fields for Shape Generation},
  author = {Cai, Ruojin and Yang, Guandao and Averbuch-Elor, Hadar and Hao, Zekun and Belongie, Serge and Snavely, Noah and Hariharan, Bharath},
  booktitle = ECCV,
  year = {2020}
}

@inproceedings{zhou2023diffusion,
  title = {Diffusion-Based Point Cloud Denoising},
  author = {Zhou, Linqi and Du, Yilun and Wu, Jiajun},
  booktitle = CVPR,
  year = {2023}
}

@article{qi2017pointnetplusplus,
  title={PointNet++: Deep Hierarchical Feature Learning on Point Sets in a Metric Space},
  author={Qi, Charles R and Yi, Li and Su, Hao and Guibas, Leonidas J},
  journal= NeurIPS,
  volume={30},
  year={2017}
}

@article{wang2019dgcnn,
  title={Dynamic Graph CNN for Learning on Point Clouds},
  author={Wang, Yue and Sun, Yongbin and Liu, Ziwei and Sarma, Sanjay E and Bronstein, Michael M and Solomon, Justin M},
  journal= TOG,
  volume={38},
  number={5},
  pages={1--12},
  year={2019},
  publisher={ACM New York, NY, USA}
}

@inproceedings{zhao2021pointtransformer,
  title={Point Transformer},
  author={Zhao, Hengshuang and Jiang, Li and Jia, Jiaya and Torr, Philip HS and Koltun, Vladlen},
  booktitle= ICCV,
  pages={16259--16268},
  year={2021}
}

@inproceedings{satorras2021equivariant,
  title={E(n) Equivariant Graph Neural Networks},
  author={Satorras, Victor Garcia and Hoogeboom, Emiel and Welling, Max},
  booktitle=ICML,
  pages={9323--9332},
  year={2021}
}

@inproceedings{chen2022equivariant,
  title={Equivariant Point Network for 3D Point Cloud Analysis},
  author={Chen, Haiwei and Liu, Shichen and Chen, Weikai and Li, Hao and Hill, Randall},
  booktitle=CVPR,
  pages={14514--14523},
  year={2021},
  doi={10.1109/CVPR46437.2021.01428}
}

@inproceedings{ma2021neuralpull,
  title={Neural-Pull: Learning Signed Distance Fields from Point Clouds by Learning to Pull Points onto Surfaces},
  author={Ma, Baorui and Han, Zhizhong and Liu, Yu-Shen and Zwicker, Matthias},
  booktitle=ICML,
  pages={7246--7257},
  year={2021}
}

@article{ma2023refined,
  title={Learning High-Quality Implicit Surfaces via Flexible Training and Refined Representation},
  author={Ma, Baorui and Liu, Yu-Shen and Zwicker, Matthias and Han, Zhizhong},
  journal={IEEE Transactions on Pattern Analysis and Machine Intelligence(TPAMI)},
  volume={45},
  number={5},
  pages={5632--5650},
  year={2023},
  publisher={IEEE}
}

@article{jiang2024straightening,
  title={Straightening Point Cloud Flows},
  author={Jiang, Jianxin and Ma, Baorui and Liu, Yu-Shen and Zwicker, Matthias and Han, Zhizhong},
  journal={arXiv preprint arXiv:2410.13233},
  year={2024}
}

@inproceedings{yu2020gradient,
	title     = {Gradient Surgery for Multi-Task Learning},
	author    = {Yu, Tianhe and Kumar, Saurabh and Gupta, Abhishek and Hausman, Karol and Levine, Sergey and Finn, Chelsea},
	booktitle = NeurIPS,
	volume    = {33},
	pages     = {5824--5836},
	year      = {2020}
}

@article{zhang2021pointfilter,
  title = {Pointfilter: Point Cloud Filtering via Encoder-Decoder Modeling},
  author = {Zhang, Dongbo and Lu, Xuequan and Qin, Hong and He, Ying},
  journal = {IEEE Transactions on Visualization and Computer Graphics},
  volume = {27},
  number = {3},
  pages = {2015--2027},
  year = {2021},
  doi = {10.1109/TVCG.2020.3027069},
  url = {https://doi.org/10.1109/TVCG.2020.3027069}
}

@inproceedings{luo2021scoredenoise,
  title = {Score-Based Point Cloud Denoising},
  author = {Luo, Shitong and Hu, Wei},
  booktitle = {Proceedings of the IEEE/CVF International Conference on Computer Vision (ICCV)},
  pages = {4583--4592},
  year = {2021},
  url = {https://openaccess.thecvf.com/content/ICCV2021/html/Luo_Score-Based_Point_Cloud_Denoising_ICCV_2021_paper.html}
}

@inproceedings{mao2022pdflow,
  title = {{PD-Flow}: A Point Cloud Denoising Framework with Normalizing Flows},
  author = {Mao, Aihua and Du, Zihui and Wen, Yu-Hui and Xuan, Jun and Liu, Yong-Jin},
  booktitle = {European Conference on Computer Vision (ECCV)},
  pages = {398--415},
  year = {2022},
  url = {https://paperswithcode.com/paper/pd-flow-a-point-cloud-denoising-framework}
}

@article{chen2023deeppsr,
  title = {Deep Point Set Resampling via Gradient Fields},
  author = {Chen, Haolan and Du, Bi'an and Luo, Shitong and Hu, Wei},
  journal = {IEEE Transactions on Pattern Analysis and Machine Intelligence},
  volume = {45},
  number = {3},
  pages = {2913--2930},
  year = {2023},
  doi = {10.1109/TPAMI.2022.3175183},
  url = {https://doi.org/10.1109/TPAMI.2022.3175183}
}

@inproceedings{desilva2023iterativepfn,
  title = {{IterativePFN}: True Iterative Point Cloud Filtering},
  author = {de Silva Edirimuni, Dasith and Lu, Xuequan and Shao, Zhiwen and Li, Gang and Robles-Kelly, Antonio and He, Ying},
  booktitle = {Proceedings of the IEEE/CVF Conference on Computer Vision and Pattern Recognition (CVPR)},
  pages = {13530--13539},
  year = {2023},
  doi = {10.1109/CVPR52729.2023.01300},
  url = {https://doi.org/10.1109/CVPR52729.2023.01300}
}

@inproceedings{vogel2024p2pbridge,
  title = {{P2P-Bridge}: Diffusion Bridges for 3D Point Cloud Denoising},
  author = {Vogel, Mathias and Tateno, Keisuke and Pollefeys, Marc and Tombari, Federico and Rakotosaona, Marie-Julie and Engelmann, Francis},
  booktitle = {European Conference on Computer Vision (ECCV)},
  year = {2024},
  url = {https://p2p-bridge.github.io/}
}

@inproceedings{guo2025asdn,
  title = {You Should Learn to Stop Denoising on Point Clouds in Advance},
  author = {Guo, Chuchen and Zhou, Weijie and Liu, Zheng and He, Ying},
  booktitle = {Proceedings of the AAAI Conference on Artificial Intelligence},
  volume = {39},
  pages = {3212--3219},
  year = {2025},
  doi = {10.1609/aaai.v39i3.32331},
  url = {https://doi.org/10.1609/aaai.v39i3.32331}
}

@inproceedings{zhou2025mambaipf,
  title = {{3DMambaIPF}: A State Space Model for Iterative Point Cloud Filtering via Differentiable Rendering},
  author = {Zhou, Qingyuan and Yang, Weidong and Fei, Ben and Xu, Jingyi and Zhang, Rui and Liu, Keyi and Luo, Yeqi and He, Ying},
  booktitle = {Proceedings of the AAAI Conference on Artificial Intelligence},
  volume = {39},
  pages = {10843--10851},
  year = {2025},
  doi = {10.1609/aaai.v39i10.33178},
  url = {https://doi.org/10.1609/aaai.v39i10.33178}
}

@inproceedings{wei2025noise2score3d,
  title = {{Noise2Score3D}: Tweedie's Approach for Unsupervised Point Cloud Denoising},
  author = {Wei, Xiangbin and Wang, Yuanfeng and Xu, Ao and Zhu, Lingyu and Sun, Dongyong and Li, Keren and Li, Yang and Qin, Qi},
  booktitle = {Proceedings of the IEEE/CVF International Conference on Computer Vision (ICCV)},
  pages = {25993--26003},
  year = {2025},
  url = {https://openaccess.thecvf.com/content/ICCV2025/html/Wei_Noise2Score3D_Tweedies_Approach_for_Unsupervised_Point_Cloud_Denoising_ICCV_2025_paper.html}
}

@article{wang2025adaptiveiterative,
  title = {Adaptive and Iterative Point Cloud Denoising with Score-Based Diffusion Model},
  author = {Wang, Zhaonan and Li, Manyi and Xin, Shiqing and Tu, Changhe},
  journal = {Computer Graphics Forum},
  volume = {44},
  number = {6},
  pages = {e70149},
  year = {2025},
  doi = {10.1111/cgf.70149},
  url = {https://doi.org/10.1111/cgf.70149}
}

@article{sheng2025dnpcd,
  title = {Deep Non-Local Point Cloud Denoising Network},
  author = {Sheng, Huankun and Li, Ying},
  journal = {Applied Soft Computing},
  volume = {171},
  pages = {112835},
  year = {2025},
  doi = {10.1016/j.asoc.2025.112835},
  url = {https://doi.org/10.1016/j.asoc.2025.112835}
}

@inproceedings{cheng2026dsnet,
  title = {Routing on Demand: {DSNet} for Efficient Progressive Point Cloud Denoising},
  author = {Cheng, Xiaoqian and Xiao, Dong and Li, Husen and Liu, Zheng and Chen, Renjie},
  booktitle = {Proceedings of the IEEE/CVF Conference on Computer Vision and Pattern Recognition (CVPR)},
  pages = {39111--39120},
  year = {2026},
  url = {https://openaccess.thecvf.com/content/CVPR2026/html/Cheng_Routing_on_Demand_DSNet_for_Efficient_Progressive_Point_Cloud_Denoising_CVPR_2026_paper.html}
}

@article{chen2026bsvpcd,
  title = {Bridging Multi-Stages and Multi-Views for Low-Quality Point Cloud Denoising},
  author = {Chen, Wu and Jiang, Qiuping and Fan, Hehe and Huang, Chao and Wang, Chunmao and Chang, Xiaojun and Yang, Yi},
  journal = {Information Fusion},
  pages = {104531},
  year = {2026},
  doi = {10.1016/j.inffus.2026.104531},
  url = {https://doi.org/10.1016/j.inffus.2026.104531}
}

@article{liu2026pointindi,
  title = {{PointInDI}: Inversion by Direct Iteration for High-Fidelity Point Cloud Denoising},
  author = {Liu, Zheng and Huang, Zhenyu and Guo, Chuchen and Pan, Maodong and He, Ying},
  journal = {Computer-Aided Design},
  pages = {104103},
  year = {2026},
  doi = {10.1016/j.cad.2026.104103},
  url = {https://doi.org/10.1016/j.cad.2026.104103}
}

@article{janowski2026nabnp,
  title = {A Learning-Free Noise-Adaptive Framework for Feature-Preserving Point Cloud Denoising},
  author = {Janowski, Artur and Karkinli, Ahmet Emin and Husrevoglu, Mustafa and Taskanat, Talha and Kesikoglu, Abdusselam},
  journal = {Applied Sciences},
  volume = {16},
  number = {13},
  pages = {6635},
  year = {2026},
  doi = {10.3390/app16136635},
  url = {https://doi.org/10.3390/app16136635}
}
